\documentclass[a4paper,10pt]{article}
\usepackage{amsfonts}
\usepackage{amsmath}
\usepackage{amsthm}
\usepackage{amssymb}
\usepackage{epsfig}

\newtheorem{defn}{Definition}[section]

\newtheorem{prop}{Proposition}[section]
\newtheorem{exm}{Example}[section]
\newtheorem{rem}{Remark}

\begin{document}
\noindent
\begin {center}
\begin{title}
\bf \textbf{Lattice Structure of Variable Precision Rough Sets} \\
\end{title} 
\end{center}
\begin{center}
Sumita Basu\\
Bethune College\\
181, Bidhan Sarani; Kolkata700006\\
sumi\underline{ }basu05@yahoo.co.in\\
\bigskip
\textbf{Abstract}
\end{center}
The main purpose of this paper is to study the lattice structure of variable precision rough sets. The notion of variation in precision of rough sets have been further extended to variable precision rough set with variable classification error      and its algebraic properties are also studied.\\


\section{Introduction} Classical rough set theory as introduced by Pawlak[1,2] is a tool for computation with data which give imprecise or vague information in terms of three valued logic. When the data set is granular in nature we are unable to observe individual objects but are forced to reason with accessible granules of knowledge. The elements of each granule can not be distinguished from the available knowledge. Due to such indiscernibility of the elements very often a subset of the entire data set(the Universal set) cannot be precisely defined. Pawlak represented the granularity by equivalence relation and defined such a set $S$ as a suitable pair of sets $ (\underline{S}, \overline{S})$ based on equivalence classes and called it a \textit{Rough Set}. It is widely used for knowledge classification [12] and rule learning [13-14].Finite state machine with rough transition have been reported in [25]. On the one hand, owing to the restrictions of equivalence relations, many researchers have presented various extensions [3-7], specially, covering-based rough sets [8 -11] are investigated as the extensions of classical rough set theory by extending partitions to coverings. On the other hand, in classical rough set model, the approximation using the classification is absolutely correct though may be somewhat imprecise. It is implied that $ \underline{S} \subset S \subset  \overline{S}$. The cardinality of the set $ BN(S)=(\overline{S} - \underline{S})$ will determine the precision of the representation. If $BN(S) = \phi$ the representation is exact.  Increase in cardinality of $B$ will increase the imprecision of the solution. \\
A generalization of rough set model was proposed by Ziarko[15]. He  introduced a measure of relative degree of misclassification(error) and chose to decrease the imprecision thereby increasing the error in approximation. This is an extension of classical rough set where the granules of knowledge are equivalence classes. Some researchers extended this concept to variable precision covering based rough set model[16-17]. \\
Algebraic properties of rough set have been  widely studied by researchers [18-24]. Algebraic properties of  variable precision rough set is discussed in this paper and it could be shown that for different classification error the set of variable precision rough sets have a lattice structure. We introduced a variable measure of degree of error and called such a set \textit{ variable precision rough set with variable error}. Properties of such sets are compared with variable precision rough sets.\\
The paper is organized as follows. In section 2 basic concepts of rough set, variable precision rough set and lattice are introduced. In section 3 structure of variable precision rough set for different classification error is explored. Section 4 is devoted to study of variable precision rough set where classification error for lower and upper approximations are not same. An example is included to explain the computation of different variable precision rough set. 
 
\section{Preliminaries}
In this section some basic concepts on \textbf{Rough Sets},  \textbf{Variable Precision Rough Sets} and \textbf{lattice} are discussed.
\subsection{Rough set }
\begin{defn} An approximation space is defined as a pair $\langle$ U, R $\rangle$, U being a non-empty set (the domain of discourse) and R an equivalence relation on it , representing indiscernibility at the object level. For $x\in R[x], R[x] $is the set of elements of $U$ indiscernible from $x$. $ E = \left\{R[x] / x\in U\right\}$ is the set of elementary blocks or defining blocks of the approximation space.\\
\end{defn}
\begin{defn} A rough set X in the approximation space $\langle $ U, R $\rangle $ is a pair ($\underline{X_{R}}$, $\overline{X_{R}}$) such that $\underline{X_{R}}$ \& $\overline{X_{R}}$ are definable sets in U defined as follows:\\
$$\underline{X}_{R} = \{{R[y]/y \in U \wedge R[y] \subseteq X}\}$$
$$\overline{X}_{R} = \{{R[y] / y \in U \wedge X \cap R[y] \neq \phi}\}$$
The region definitely belonging to $X$ is denoted by $D(X)$ and defined by $ D(X) =\underline{ X}_R$. The boundary region $BN(X)$ of the rough set X is  $\underline{X}_{R}$ - $\overline{X}_{R}$. The region not included in $X$ is denoted by $ N(X) $ and defined by $U-\overline{X}_R $.
\end{defn}
\begin{defn} The accuracy of approximation by the rough set X is given by
$$  \alpha = \frac{card(\underline{X})}{card(\overline{X})}$$
\end{defn}
\begin{rem}
 If for a rough set $X$, $\underline{X}_{R}$ = $\overline{X}_{R}$, i.e $B_R = \phi$ then the rough set is precisely defined and the accuracy of approximation is 1.In general the accuracy of approximation is $\alpha \in \left[0,1\right]$
\end{rem}
\subsection{Variable Precision Rough Set}
In this rough set model a set $X \subseteq U$ is approximately defined using three exactly definable sets : $D(X), BN(X) ~and~ N(X)$. However, it may so happen that an elementary set $R[y]$ where $y\in U$ is such that although  $R[y]\cap D(X)=\phi, card(R[y]\cap X)$ is quite high relative to $card(R[y])$. So inclusion of $R[y]$ in $D(X)$ will incur a small amount of error. However, if we agree to accept this error we will be able to increase the precision of the rough set so obtained. With this idea Ziarko formulated Variable Precision Rough Set(VPRS) which is defined below.
\begin{defn}
A measure of the degree of overlap between two sets X and Y with respect to X is denoted by d(X,Y) and defined by,
$$d(x,y) = 1-\frac{ card (X\cap Y)}{ card(X)}$$
\end{defn}
\begin{defn}
A variable precision rough set(VPRS) $X(\beta)$ in the approximation space $\langle$ U, R $\rangle $, is a pair ($\underline{X}_{R}(\beta)$, $\overline{X}_{R}(\beta)$) such that $\underline{X}_{R}(\beta)$ \& $\overline{X}_{R}(\beta)$ are definable sets in U  defined as follows:
$$\underline{X}_{R}(\beta) = \{{R[y]/y\in U \wedge R[y] \subset  X \wedge d(R[y],X)\leq \beta}\}$$
$$\overline{X}_{R}(\beta) = \{{R[y]/y \in U \wedge X \cap R[y] \neq\phi \wedge d(R[y],X) \leq1-\beta}\}$$
For the variable precision rough set model with $\beta$ error a set $X \subseteq U$ is approximately defined using three sets of definable sets : $DX(\beta), BNX(\beta) ~and~ NX(\beta)$as follows:
$$DX(\beta )=\{{R[y]/y\in U \wedge R[y] \subset  X \wedge d(R[y],X)\leq\beta}\}$$
$$BNX(\beta)= \overline{X}_{R}(\beta)-\underline{X}_{R}(\beta)$$
$$NX(\beta) = \{{R[y]/y\in U \wedge X \cap R[y] \neq\phi \wedge d(R[y],X)  >1- \beta}\}$$
\end{defn}
In general, $\beta$ is chosen so that $\beta \in [0, .5)$. For given $X, \beta$, $DX(\beta)$ is the region \textit{included} in $X$, $NX(\beta)$ is the region \textit{not included }in $X$and $BNX(\beta)$ is the boundary region \textit{possibly included }in $X$. If $BNX(\beta)= \phi$ then $X$ is $\beta$discernible.
\begin{defn} The accuracy of approximation by the rough set $X(\beta)$ is given by
$$  \alpha X(\beta) = \frac{card(\underline{X}_{R}(\beta))}{card(\overline{X}_{R}(\beta))}$$
\end{defn}
\begin{prop}Let $X$ be an arbitrary subset of the universe $U$  in the approximation space $\langle U, R \rangle$, and $\beta$ be the error specified then,
\begin{enumerate}
\item $DX(\beta) \cup BNX(\beta) \cup NX(\beta) = U $
\item $DX(\beta) \cap BNX(\beta) = BNX(\beta) \cap NX(\beta) =DX(\beta) \cap NX(\beta) = \phi$
\end{enumerate}
\end{prop}
\begin{prop}Let $X$ be an arbitrary subset of the universe $U$  in the approximation space $\langle U, R \rangle$, and $\beta_1 < \beta_2$ then,
\begin{enumerate}
\item $\underline{X}_{R}(\beta_1) \subseteq \underline{X}_{R}(\beta_2) $
\item $\overline{X}_{R}(\beta_1) \subseteq \overline{X}_{R}(\beta_2) $
\item $ DX(\beta_1)\subseteq DX(\beta_2) $
\item $ NX(\beta_1)\subseteq NX(\beta_2) $
\item $ BNX(\beta_2)\subseteq BNX(\beta_1) $
\item $\alpha X(\beta_1)\leq \alpha X(\beta_2)$
\end{enumerate}
\end{prop}
\begin{proof}Results 1-5 follows from the definition2.5.Result 6 follows from Result1 and 2.
\end{proof}
\subsection{Lattice}
\begin{defn}Let $L$ be a set of elements in which the binary operations  $\bigcap, \bigcup$ (meet and joint respectively) and $=$(equality) are defined. An algebra $L=\langle L, \bigcap, \bigcup\rangle$ is a lattice if the following identities are true in $L$. Let $x,y,z \in L $
\begin{enumerate}
\item Idempotence: $x\bigcup x =x;~~~~x\bigcap x=x$
\item Commutativity: $x\bigcup y =y \bigcup x;~~~~x\bigcap y=y \bigcap x$
\item Associativity: $x\bigcup (y\bigcup z) = (x\bigcup y)\bigcup z;~~~~x\bigcap( y\bigcap z) =(x \bigcap y)\bigcap z$
\item Absorption: $x\bigcup(x\bigcap y ) = x;~~~~x\bigcap (x\bigcup y = x$
\end{enumerate}
\end{defn}
\section{Order in VPRS with respect to Classification Error $\beta$}
\begin{prop} For an arbitrary subset $X$ of the universe $U$, let us define $\textbf{D}X=\{DX(\beta) / \beta \in [0, .5)\}$,then $\textbf{D}X$ is a totally orderd set with $DX(0) = \underline{X}_R$ as the least element and $DX(0.5)$ as the greatest element.
\end{prop}
This result follows from 3 of proposition 2.2. Similarly we have the following propositions:
\begin{prop} For an arbitrary subset $X$ of the universe $U$, let us define $\textbf{N}X=\{NX(\beta) / \beta \in [0, .5)\}$,then $\textbf{N}X$ is a totally orderd set with $NX(0) = U- \overline{X}_R$ as the least element and $NX(0.5)$ as the greatest element.
\end{prop}
\begin{prop} For an arbitrary subset $X$ of the universe $U$, let us define $\textbf{BN}X=\{BNX(\beta) / \beta \in [0, .5)\}$,then $\textbf{BN}X$ is a totally orderd set with $BNX(0) = \overline{X}_R -\underline{X}_R$ as the greatest element and $BNX(0.5)$ as the least element.
\end{prop}
\begin{defn}Let $\tilde{X}$ be the set of all VPRS for $X\subset U$ where the classification error $ \beta \in [0, .5)$. So, $\tilde{X}=\{X(\beta)/ \beta \in [0, .5)\}$so that $X(\beta)= (\underline{X}_{R}(\beta),\overline{X}_{R}(\beta))$
\end{defn}
\begin{defn}Let $X(\beta_1),X(\beta_2) \in \tilde{X}$ then $X(\beta_1)\subseteq X(\beta_2) $ iff $\underline{X}_{R}(\beta_1) \subseteq\underline{X}_{R}(\beta_2)~and~\overline{X}_{R}(\beta_1) \subseteq\overline{X}_{R}(\beta_2) $.
\end{defn}
\begin{defn}Let $B=\{\beta_i/ \beta_i\in [0, .5)~and~(i\leq j \rightarrow (\beta_i\leq \beta_j)\}$. Then $B$ is a totally ordered set.
\end{defn}
\begin{prop} For an arbitrary subset $X$ of the universe $U$,
\begin{enumerate}
\item $\{\underline{X}_{R}(\beta_i)/\beta_i \in B\}$ is a totally ordered set with lub$\{\underline{X}_{R}(\beta_i)\} =\underline{X}_{R}(\beta_{0.5})$ and glb$\{\underline{X}_{R}(\beta_i)\} =\underline{X}_{R}(0)=\underline{X}_R$
\item $\{\overline{X}_{R}(\beta_i)/\beta_i \in B\}$ is a totally ordered set with glb$\{\overline{X}_{R}(\beta_i)\} =\overline{X}_{R}(\beta_{0.5})$ and lub$\{\overline{X}_{R}(\beta_i)\} =\overline{X}_{R}(0)=\overline{X}_R$
\item $\underline{X}_R(0.5)\subseteq \overline{X}_R(0.5)$
\end{enumerate}
\end{prop}
\begin {prop} If, $ \bigcup ~and ~\bigcap $ represent the union and intersection operation of two sets then we have the following:
\begin{enumerate}
\item$\underline{X}_{R}(\beta_i) \bigcup\underline{X}_{R}(\beta_j)= \underline{X}_{R}(\beta_j)~if~  \beta_i \leq \beta_j$
\item$\underline{X}_{R}(\beta_i) \bigcap\underline{X}_{R}(\beta_j)= \underline{X}_{R}(\beta_i)~if~  \beta_i \leq \beta_j$
\item$\overline{X}_{R}(\beta_i) \bigcup\overline{X}_{R}(\beta_j)= \underline{X}_{R}(\beta_i)~if~  \beta_i \leq \beta_j$
\item$\overline{X}_{R}(\beta_i) \bigcap\overline{X}_{R}(\beta_j)= \overline{X}_{R}(\beta_j)~if~  \beta_i \leq \beta_j$
\end{enumerate}
\end{prop}
\begin{defn}Two binary operations joint ($\bigcup$) and meet ($\bigcap$) are defined on $\tilde{X}$ as follows:
$$X(\beta_1)\bigcup X(\beta_2) = (( \underline{X_{R(\beta_1)}} \bigcup\underline{X_{R(\beta_2)} }), ( \overline{X_{R(\beta_1)}} \bigcup\overline{X_{R(\beta_2)} }))$$
$$X(\beta_1)\bigcap X(\beta_2)= (( \underline{X_{R(\beta_1)}} \bigcap\underline{X_{R(\beta_2)} }),(\overline{X_{R(\beta_1)}} \bigcap\overline{X_{R(\beta_2)} }) )$$
\end{defn}
\begin{defn}Two VPRS $X(\beta_i)~and~X(\beta_j)$ are said to be equal if $\underline{X}_R(\beta_i) = \underline{X}_R(\beta_j)$ and $\overline{X}_R(\beta_i) = \overline{X}_R(\beta_j)$
\end{defn}
The approximation space remaining the same the equivalence relation $R$ will remain the same and henceforth R will not be mentioned explicitly.
\begin{prop}If,$X(\beta_i), X(\beta_j)\in \tilde{X}$ then 
\begin{enumerate}
\item$X(\beta_i)\bigcup X(\beta_j)=(\underline{X}(\beta_j), \overline{X}(\beta_i)) ~if~ \beta_i\leq \beta_j $ 
\item$X(\beta_i)\bigcap X(\beta_j)=(\underline{X}(\beta_i), \overline{X}(\beta_j)) ~if~ \beta_i\leq \beta_j $
\end{enumerate}
\end{prop}
\begin{prop}Binary operations $\bigcap~and~\bigcup$ are idempotent and commutative in $\tilde{X}$
\end{prop}
\begin{proof}
From 1 of Prop 3.6,$$X(\beta_i)\bigcup X(\beta_i)=(\underline{X}(\beta_i), \overline{X}(\beta_i))=X(\beta_i)$$
Also,$$X(\beta_i)\bigcup X(\beta_j)=(\underline{X}(\beta_j), \overline{X}(\beta_i))=X(\beta_j)\bigcup X(\beta_i) ~if~ \beta_i\leq \beta_j $$
The result for $\bigcap$ may be proved similarly.
\end{proof}
\begin{prop}Binary operations $\bigcap~and~\bigcup$ are associative in $\tilde{X}$. 
\end{prop}
\begin{proof}
\begin{equation}\label{xx}
\begin{split}
X(\beta_i)\bigcap ( X(\beta_j)\bigcap X(\beta_k))&=(X(\beta_i)\bigcap  X(\beta_j))\bigcap X(\beta_k)\\
&=(\underline{X}(\beta_i), \overline{X}(\beta_k))~if~\beta_i \leq \beta_j \leq \beta_k\\
&=(\underline{X}(\beta_i), \overline{X}(\beta_j))~if~\beta_i \leq \beta_k \leq \beta_j\\
&=(\underline{X}(\beta_k), \overline{X}(\beta_j))~if~\beta_k \leq \beta_i \leq \beta_j\\
&=(\underline{X}(\beta_j), \overline{X}(\beta_k))~if~\beta_j \leq \beta_i \leq \beta_k\\
&=(\underline{X}(\beta_j), \overline{X}(\beta_i))~if~\beta_j \leq \beta_k \leq \beta_i\\
&=(\underline{X}(\beta_k), \overline{X}(\beta_i))~if~\beta_k \leq \beta_j \leq \beta_i\\
\end{split}
\end{equation}
Hence the $\bigcap$ operation is associative. Similarly it can be shown that the $\bigcup$ operation is associative.
\end{proof}
\begin{prop}For the binary operations $\bigcap~and~\bigcup$ absorption rule hold in $\tilde{X}$.So,
$$X(\beta_i)\bigcap(X(\beta_i)\bigcup X(\beta_j))=X(\beta_i);~~~~~X(\beta_i)\bigcup(X(\beta_i)\bigcap X(\beta_j))=X(\beta_i);~~if~\beta_i,\beta_j \in B$$
\end{prop}
\begin{proof}Case I:  $\beta_i \leq \beta_j $\\
$$X(\beta_i)\bigcap ( X(\beta_i)\bigcup X(\beta_j))=(\underline{X}(\beta_i), \overline{X}(\beta_i))\bigcap (\underline{X}(\beta_j), \overline{X}(\beta_i))=(\underline{X}(\beta_i), \overline{X}(\beta_i))= X(\beta_i)$$
Case II:  $\beta_j\leq \beta_i$\\
$$X(\beta_i)\bigcap ( X(\beta_i)\bigcup X(\beta_j))=(\underline{X}(\beta_i), \overline{X}(\beta_i))\bigcap (\underline{X}(\beta_i), \overline{X}(\beta_j))=(\underline{X}(\beta_i), \overline{X}(\beta_i))= X(\beta_i)$$
The other part may be similarly proved.
\end{proof}
Using Propositions 3.7, 3.8 and 3.9 we get the final result.
\begin{prop}$(\tilde{X},\bigcup,\bigcap) $form a lattice.
\end{prop}
\section{VPRS with Variable Classification Error $(\beta, \gamma)$}
Discussions of VPRS show that both lower and upper approximations vary with classification error. It may so happen that for a particular problem the error admissible for the lower approximation and the error admissible for the upper approximation are different. The variable precision rough set with variable error is defined below.
\begin{defn}
A variable precision rough set with variable error(VPRSVE) X($\beta ,\gamma $) in the approximation space $\langle$ U, R $\rangle $, is a pair ($\underline{X}_{R}(\beta,\gamma), \overline{X}_{R}(\beta,\gamma)$) such that $\underline{X}_{R}(\beta, \gamma)~\&~ \overline{X}_{R}(\beta, \gamma)$ are definable sets in U  defined as follows:
$$\underline{X}_{R}(\beta, \gamma) =\underline{X}_{R}(\beta) = \{{R[y]/y\in U \wedge R[y] \subset  X \wedge d(R[y],X)\leq \beta}\}$$
$$\overline{X}_{R}(\beta, \gamma) =\overline{X}_{R}(\gamma) = \{{R[y]/y\in U \wedge X \cap R[y] \neq\phi \wedge d(R[y],X)\leq(1-\gamma)}\}$$
For the VPRSVE with $(\beta,\gamma)$ error a set $X \subseteq U$ is approximately defined using three sets of definable sets : $DX(\beta, \gamma), BNX(\beta,\gamma) ~and~ NX(\beta, \gamma)$as follows:
$$DX(\beta, \gamma )=\{{R[y]/y\in U \wedge R[y] \subset  X \wedge d(R[y],X)\leq\beta}\}$$
$$BNX(\beta, \gamma)= \overline{X}_{R}(\gamma)-\underline{X}_{R}(\beta)$$
$$NX( \beta, \gamma) = \{{R[y]/y\in U \wedge X \cap R[y] \neq\phi \wedge d(R[y],X) > (1- \gamma)}\}$$
\end{defn}
\begin{rem}According to the requirement of the situation the boundary region of the VPRSVE $X(\\beta, gamma)$(denoted by $BNX(\beta,\gamma)= \overline{X}_R(\gamma)-\underline{X}_R(\beta)$)is increased or decreased.
\end{rem}
Proposition 2.1 will be modified in this case as 
\begin {prop}Let $X$ be an arbitrary subset of the universe $U$  in the approximation space $\langle U, R \rangle$, and $\beta, \gamma \in [0, 0.5) $ be the error specified then,
\begin{enumerate}
\item $DX(\beta, \gamma) = DX(\beta)$
\item $NX(\beta, \gamma) = NX(\gamma)$
\item $DX(\beta, \gamma) \cup BNX(\beta, \gamma) \cup NX(\beta, \gamma) = U $
\item $DX(\beta, \gamma) \cap BNX(\beta, \gamma) = BNX(\beta, \gamma) \cap NX(\beta, \gamma) =DX(\beta, \gamma) \cap NX(\beta, \gamma) = \phi$
\end{enumerate}
\end{prop}
\begin{exm}Let $U=\{x_i/i=1,2,3.....25\}$ and R is an equivalence relation on U such that $ U/ R = \{ [x_1], [x_2, x_3 ],[x_4,x_5,x_6],[x_7,x_8], [x_9], [x_{10}, x_{11}], [x_{12},x_{13},x_{14},x_{15}],[x_{16}], [x_{17}], [x_{18},x_{19},x_{20}],[x_{21},x_{22},x_{23},x_{24}], [x_{25}]\}$. Let $A=\{x_3, x_4, x_5, x_{10}, x_{11},x_{13},x_{14}, x_{15}, x_{19}, x_{21}\}$\\ Problem: Define A with respect to the equivalence classes of U/R\\\\ Pawlakian rough set $A =( \underline{ A},\overline {A})$ where$\underline {A}=\{[x_{10}, x_{11}]\}$ and \\$\overline { A}=\{[x_{10}, x_{11}], [x_2, x_3],[x_4, x_5, x_6], [x_{12},x_{13},x_{14},x_{15}], [x_{18},x_{19},x_{20}], [x_{21},x_{22},x_{23},x_{24}] \}$,so that $$DA=\{[x_{10}, x_{11}]\}$$ $$BNA=\{[x_2, x_3],[x_4, x_5, x_6], [x_{12},x_{13},x_{14},x_{15}], [x_{18},x_{19},x_{20}], [x_{21},x_{22},x_{23},x_{24}] \}$$ $$NA=\{ [x_1], [x_7,x_8], [x_9],[x_{16}], [x_{17}], [x_{25}]\}$$ For VPRS $A, \beta$  can have values $0.25, 0.33,0.5.$ So there can be three possible VPRS $A(0.25), A(0.33), A(0.5)$. Thus,$$\underline{A}(0.25)=\{[x_{10}, x_{11}], [x_{12},x_{13},x_{14},x_{15}]\}$$
$$\overline{A}(0.25)=\{[x_{10}, x_{11}], [x_2, x_3],  [x_{18},x_{19},x_{20}], [x_{21},x_{22},x_{23},x_{24}],[x_4, x_5, x_6], [x_{12},x_{13},x_{14},x_{15}]\}$$
$$DA(0.25)=\{[x_{10}, x_{11}],  [x_{12},x_{13},x_{14},x_{15}]\}$$ 
$$BNA(0.25)=\{[x_2, x_3],[x_4, x_5, x_6],  [x_{18},x_{19},x_{20}], [x_{21},x_{22},x_{23},x_{24}]\}$$ 
$$NA(0.25)=\{ [x_1 ], [x_7,x_8], [x_9],[x_{16}], [x_{17}], [x_{25}]\}$$Also, $$\underline{A}(0.33)=\{[x_{10}, x_{11}], [x_4, x_5, x_6], [x_{12},x_{13},x_{14},x_{15}]\}$$
$$\overline{A}(0.33)=\{[x_{10}, x_{11}], [x_2, x_3],  [x_{18},x_{19},x_{20}],[x_4, x_5, x_6], [x_{12},x_{13},x_{14},x_{15}]\}$$
$$DA(0.33)=\{[x_{10}, x_{11}], [x_4, x_5, x_6], [x_{12},x_{13},x_{14},x_{15}]\}$$ 
$$BNA(0.33)=\{ [x_2, x_3], [x_{18},x_{19},x_{20}]\}$$
 $$NA(0.33)=\{ [x_1], [x_7, x_8], [x_9],[x_{16}], [x_{17}], [x_{21},x_{22},x_{23},x_{24}], [x_{25}]\}$$and,
$$\underline{A}(0.5)=\{[x_{10}, x_{11}], [x_2, x_3],[x_4, x_5, x_6],  [x_{12},x_{13},x_{14},x_{15}]\}= \overline {A}(0.5)$$
$$DA(0.5)=\{[x_{10}, x_{11}], [x_2, x_3],[x_4, x_5, x_6],  [x_{12},x_{13},x_{14},x_{15}]\}$$ 
$$BNA(0.5)=\phi $$ 
$$NA(0.5)=\{ [x_1 ], [x_7,x_8], [x_9],[x_{16}], [x_{17}],[x_{21},x_{22},x_{23},x_{24}],  [x_{25}], [x_{18},x_{19},x_{20}]\}$$
Six VPRSVE are possible for $A$ defined with respect to given approximation space of which $A(0.25,0.33)$ is given below:$$\underline{A}(0.25,0.33)=\underline{A}(0.25)=\{[x_{10}, x_{11}], [x_{12},x_{13},x_{14},x_{15}]\}$$
$$\overline{A}(0.25,0.33)=\overline{A}(0.33)=\{[x_{10}, x_{11}], [x_2, x_3],  [x_{18},x_{19},x_{20}],[x_4, x_5, x_6], [x_{12},x_{13},x_{14},x_{15}]\}$$
$$DA(0.25,0.33)=\{[x_{10}, x_{11}], [x_{12},x_{13},x_{14},x_{15}]\}$$
$$BNA(0.25,0.33)=\{[x_2, x_3],[x_4, x_5, x_6],  [x_{18},x_{19},x_{20}]\}$$
$$NA(0.25,0.33)=\{ [x_1], [x_7, x_8], [x_9],[x_{16}], [x_{17}], [x_{21},x_{22},x_{23},x_{24}], [x_{25}]\}$$
\end{exm}
 \section{Conclusion} In this paper algebraic properties of set of VPRS for a particular imprecise set X have been studied. In order to define such an imprecise set the approximation space is partitioned into three regions,the included region($DX(\beta)$), the boundary region($BNX(\beta)$) and the rejection region($NX(\beta)$).For a particular X with variations of $\beta$ the regions vary. It could also be shown that the set of all VPRS for the set X forms a lattice. We extended the classification error $\beta$ to a pair $( \beta, \gamma)$ and explained its use with an example. The included region,boundary region and rejection region for a VPRSVE is defined and it is shown that these three regions partition the approximation space. Study of the algebraic properties of VPRSVE is an open area of research.

\end{document}